%% file: mixture.tex
\documentclass[11pt]{article}

\input{macros}

\begin{document}
\title{Efficient Orthogonal Tensor Decomposition,\\ with an Application to Latent Variable Model Learning}
\author{
	\href{http://www.ucl.ac.uk/statistics/people/franz-kiraly}{Franz J. Király} \thanks{Department of Statistical Science, University College London; and MFO; \emailaddr{f.kiraly@ucl.ac.uk}}}
\date{}
\maketitle
\begin{abstract}
\begin{normalsize}
Decomposing tensors into orthogonal factors is a well-known task in statistics, machine learning, and signal processing. We study orthogonal outer product decompositions where the factors in the summands in the decomposition are required to be orthogonal across summands, by relating this orthogonal decomposition to the singular value decompositions of the flattenings. We show that it is a non-trivial assumption for a tensor to have such an orthogonal decomposition, and we show that it is unique (up to natural symmetries) in case it exists, in which case we also demonstrate how it can be efficiently and reliably obtained by a sequence of singular value decompositions. We demonstrate how the factoring algorithm can be applied for parameter identification in latent variable and mixture models.
\end{normalsize}
\end{abstract}

\section{Introduction}
Decomposing a tensors into its components, and determining the number of those (= the rank) is a multidimensional generalization of the singular value decomposition and the matrix rank, and a reoccurring task in all practical sciences, appearing many times under different names; first discovered by Hitchcock~\cite{Hitchcock1927} and then re-discovered under names such as PARAFAC~\cite{Harshman1970} or CANDECOMP~\cite{CarrollChang1970}, it has been applied in many fields such as chemometrics, psychometrics, and signal processing \cite{Bro1997,Parafac2000,NionSidi2009}. An extensive survey of many applications can be found in \cite{Sidi2004,DeLathauwer:2008}.\\

Recently, motivated by real world applications, orthogonality constraints on the decomposition have been studied in the literature, such as the orthogonal rank decomposition and the combinatorial orthogonal rank decomposition, which can be traced back to~\cite{Deni1989,Kolda01orthogonaltensor}, and the orthogonal decomposition in~\cite{Martin06ajacobi-type} and~\cite{Hsu2012}, the latter of which occurs for example in the identification of latent variable models from empirical moments, and several other statistical estimation tasks, see~\cite{Ana2012} for a survey. The orthogonality constraints imposed in these two branches of literature are not the same, as~\cite{Deni1989,Kolda01orthogonaltensor} imposes summand-wise orthogonality, while in~\cite{Martin06ajacobi-type,Hsu2012,Ana2012}, factor-wise orthogonality can be deduced from the model constraints. In~\cite{Martin06ajacobi-type}, a Jacobi-like and heuristic algorithm was described to obtain a close orthogonal decomposition via Jacobi angle optimization for general tensors; in~\cite{Ana2012}, the authors describe a second order fixed point method for obtaining the decomposition.\\

In~\cite{Ish2013ICML,Son2013ICML}, hierarchical tensor decomposition models are discussed in the context of latent tree graphical models, and algorithms for the identification of this decomposition are described. While this is not explicitly done in the language of orthogonal tensor decompositions, the idea of using flattenings is similar to the one presented, and, in the specific context of tree models, a specific instance orthogonal tensor decomposition, as described in~\cite{Ana2012}.\\

In this paper, we study the orthogonal decomposition model, as it occurs in~\cite{Hsu2012,Ana2012}, namely with factor-wise orthogonality constraints. We show that this kind of decomposition can be directly transformed to a set of singular value decompositions, both theoretically and practically. We give identifiability results for this kind of orthogonal decomposition, showing that it is unique\footnote{up to natural symmetries} in case of existence, and we provide algorithms to obtain the orthogonal decomposition, by reducing it to a sequence of singular value decompositions. We apply these algorithms to a latent variable identification problem which was discussed in~\cite{Hsu2012,Ana2012}, reducing it to a series of eigenvalue problems. In particular, by performing the reduction to singular value decomposition, we show that all existing theory on the singular value decomposition, concerning theoretical issues as well as numerical and algorithmical ones, can be readily applied to the orthogonal decomposition problem.

\section{Theoretical Background}
\subsection{Tensors}
\subsubsection{Definition of a Tensor}
While tensors are common objects, their notation diverges throughout the literature. For ease of reading, we provide the basic definitions.
\begin{Def}
A real tensor of size $(n_1\times n_2\times\dots\times n_d)$ and of degree $d$ is an element of the set
$$\RR^{n_1\times n_2\times \dots \times n_d}=\left\{(a_{i_1\dots i_d})_{\begin{subarray}{l}
        1\le i_1\le n_1\\  \vdots \\ 1\le i_d\le n_d
      \end{subarray}}\right\}.$$
If $n_1 = n_2 = \dots = n_d,$ we also write $\RR^{n^{\times d}}:= \RR^{n_1\times n_2\times \dots \times n_d}.$
\end{Def}

\subsubsection{Linear Transformation}
Let us introduce a useful shorthand notation for linearly transforming tensors.
\begin{Def}
Let $A\in \RR^{m\times n}$ be a matrix. For a tensor $T\in \RR^{n^{(\times d)}}$, we denote by $A\circ T$ the application of $A$ to $T$ along all tensor dimensions, that is, the tensor $A\circ T\in \RR^{m^{(\times d)}}$ defined as
$$\left(A\circ T\right)_{i_1\dots i_d}=\sum_{j_1=1}^n\dots \sum_{j_d=1}^n A_{i_{1}j_{1}}\cdot\ldots\cdot A_{i_{d}j_{d}}\cdot T_{j_1\dots j_d}.$$
\end{Def}

\begin{Rem}
For $T\in \RR^{n^{(\times d)}}$ and $A\in \RR^{m\times n}, A'\in \RR^{m'\times m}$, note that
$$A'\circ (A\circ T) = (A'\cdot A) \circ T.$$
\end{Rem}

\subsubsection{Flattening}
A flattening of a tensor is the tensor obtained from regarding different indices as one index.

\begin{Def}
Denote by $[k]=\{1,2,\dots, k$\}. A surjective map $\sigma:[d]\rightarrow [\tilde{d}]$ is called $d$-to-$\tilde{d}$ \emph{flattening map}.\\
\end{Def}

\begin{Def}
Let $T\in \RR^{n_1\times \dots \times n_d}$ be a tensor, and let $\sigma$ be a $d$-to-$\tilde{d}$ flattening map. Then, the $\sigma$-flattening of $T$ is the degree $\tilde{d}$ tensor $\sigma\dashv T\in \RR^{\tilde{n}_1\times \dots \times \tilde{n}_{\tilde{d}}},$ with $\tilde{n}_k=\prod_{\ell\in\sigma^{-1}(k)}n_\ell,$ defined as
$$(\sigma\dashv T)_{j_1\dots j_{\tilde{d}}} := T_{i_1\dots i_d}\quad,\mbox{where}\; j_k=(i_\ell\;:\;\ell\in \sigma^{-1}(k)).$$
Conversely, if $\tilde{T}=\sigma\dashv T$, then we write $T=\sigma\vdash\tilde{T}$ and call $T$ the \emph{unflattening} of $\tilde{T}$.
\end{Def}
Note that the indices of $\sigma\dashv T$ are, as defined, tuples of indices of $T$; however, this does not contradict the definition of tensor since $[n_1]\times[n_2]\times \dots [n_k]$ can be bijectively mapped onto $\left[\prod_{i=1}^k n_i\right].$ It is convenient to choose the lexicographical ordering for the bijection, but it is mathematically not necessary to fix any such bijection.

For unflattening, if only $\tilde{T}$ and $\sigma$ are given, it is not clear what $\sigma\vdash\tilde{T}$ should be without further specification, since the same unflattening can arise from different tensors even if $\sigma$ is fixed. Therefore, we will use it only in the context where a given flattening is being reversed, or partially reversed, therefore making the unflattening well-defined.

\begin{Ex}
Let $T\in \RR^{n_1\times n_2\times n_3}$ be a tensor, let $\sigma: 1\mapsto 1, 2\mapsto 2, 3\mapsto 2$. The $\sigma$-flattening of $T$ is a $(n_1\times n_2n_3)$-matrix $\tilde{T}:=\sigma\dashv T$. The columns of $\sigma\dashv \tilde{T}$ are all the $n_2n_3$ sub-$(n_1\times 1\times 1)$-tensors of $T$ where second and third index are fixed. The columns of $\sigma\dashv T$ are indexed by the pairs $(k,\ell)$, or, alternatively, by bijection, by the lexicographical index number $(k-1)\cdot n_2 + \ell$. Taking any $(n'_1\times n_2n_3)$-submatrix of $\tilde{T}$, we can unflatten to obtain a $(n'_1\times n_2\times n_3)$-tensor $\sigma\vdash \tilde{T}$.
\end{Ex}

\subsubsection{Outer Product}
Furthermore, we introduce notation for creating tensors of higher order out of tensors of lower order:
\begin{Def}
Let $v^{(1)}\in \RR^{n_1},\dots, v^{(d)}\in \RR^{n_d}$. The \emph{outer product} of the $v^{(k)}$ is the tensor $v^{(1)}\otimes \dots \otimes v^{(d)}\in\RR^{n_1\times \dots n_d}$ defined by
$$(v^{(1)}\otimes \dots \otimes v^{(d)})_{i_1\dots i_d} := \prod_{k=1}^d v^{(k)}_{i_k}.$$
In case that $v=v^{(1)}= \dots = v^{(d)}$, we also write $v^{\otimes d} := v^{(1)}\otimes \dots \otimes v^{(d)}.$\\

Similarly, if $A\in \RR^{n_1\times \dots\times n_c}$ and $B\in \RR^{n_{c+1}\times \dots\times n_d}$ are tensors, the outer product of $A$ and $B$ is the tensor $A\otimes B\in \RR^{n_1\times \dots n_d}$ defined as
$$(A\otimes B)_{i_1\dots i_d} := \prod_{k=1}^c A^{(k)}_{i_1\dots i_c}\cdot \prod_{k=c+1}^d B^{(k)}_{i_{c+1}\dots i_d}.$$
Outer products of several tensors $A_1\otimes \dots\otimes A_k$ %and powers $A^{\otimes d}$ can be
by induction on $k$, namely:
$$A_1\otimes \dots\otimes A_k := (A_1\otimes \dots\otimes A_{k-1})\otimes A_k.$$
\end{Def}

A useful calculation rule for linear transformation is the following:
\begin{Lem}
Let $A\in \RR^{n^{\times d_1}}$ and $B\in \RR^{n^{\times d_2}},$ let $A\in\RR^{m\times n}$. Then,
$$P\circ (A\otimes B) = (P\circ A)\otimes (P\circ B).$$
Similarly, if $v\in \RR^{n}$, then $P\circ \left(v^{\otimes d}\right) = \left(P\circ v\right)^{\otimes d}.$
\end{Lem}

Outer products are also compatible with flattenings:
\begin{Lem}\label{Lem:prodflat}
Let $A\in\RR^{n_1\times \dots\times n_c}$ and $B\in\RR^{n_{c+1}\times \dots\times n_d}.$ Let $\tau$ be a $d$-to-$k$-flattening, let $\sigma_1$ be the restriction of $\tau$ to $[c]$, and let $\sigma_2$ be the $(d-c)$-to-$\tilde{k}$-flattening defined by $\sigma(i):=\tau(c+i)$. Then,
$$\tau\dashv(A\otimes B) = (\sigma_1\dashv A)\otimes (\sigma_2\dashv B).$$
\end{Lem}

%\subsubsection{Tensor Decomposition}
%The construction operator, which is the reverse of tensor decomposition, is the following generalized form of matrix multiplication:
%\begin{Def}
%Let $A^{(1)},\dots, A^{(d)}\in \RR^{n_d\times r}$ be matrices. We call $\star \left[A^{(1)},\dots, A^{(d)}\right]$ the \emph{sum-product} of the $A^{(i)}$, defined as the tensor $\star \left[A^{(1)},\dots, A^{(d)}\right]\in \RR^{n_1,\dots, n_d},$ defined as
%$$\star \left[A^{(1)},\dots, A^{(d)}\right]_{i_1\dots i_d}=\sum_{k=1}^r A^{(1)}_{i_1 k} \cdot A^{(d)}_{i_d k}.$$
%Moreover, if $A=A^{(1)} =\dots =A^{(d)}$, we also write $A^{\star d}:=\star \left[A^{(1)},\dots, A^{(d)}\right].$
%\end{Def}
%
%\begin{Rem}
%For $v^{(1)}\in \RR^{n_1},\dots, v^{(d)}\in \RR^{n_d},$ it holds that
%$$v^{(1)}\otimes \dots \otimes v^{(d)}=\star \left[v^{(1)},\dots, v^{(d)}\right].$$
%\end{Rem}

%\begin{Def}
%Let $T\in \RR^{n_1\times n_2\times \dots \times n_d}$. A sum-product decomposition
%$$T=\star \left[A^{(1)},\dots, A^{(d)}\right],$$
%where $A^{(1)},\dots, A^{(d)}\in \RR^{n_d\times r}$, is called rank-$r$ CP-decomposition of $T$.\\
%
%If $r\le n_k$ for $k\lneq d$, $A^{(k)}$ is row-orthonormal for $k\lneq d$, and $n_d=1$, the decomposition is called an orthogonal decomposition.
%\end{Def}

\subsection{Orthogonality and Duality}

We briefly review the notions of scalar product and some results, which can also be found in~\cite{Kolda01orthogonaltensor} in slightly different formulation and slightly less generality.
\begin{Def}
A \emph{scalar product} is defined on $\RR^{n_1\times n_2\times \dots \times n_d}$ by
\begin{align*}
 \langle .,.\rangle:& \RR^{n_1\times n_2\times \dots \times n_d}\times \RR^{n_1\times n_2\times \dots \times n_d}\longrightarrow\RR\\
 & (A,B)\mapsto \sum_{i_1=1}^{n_1}\dots \sum_{i_d=1}^{n_d} A_{i_1\dots i_d}\cdot B_{i_1\dots i_d}
\end{align*}
As usual, $A,B\in \RR^{n_1\times \dots \times n_d}$ are called orthogonal to each other if $\langle A,B\rangle =0$, and $A$ is called normal if $\langle A,A\rangle = 1$. A set $A_1,\dots, A_r\in \RR^{n_1\times \dots \times n_d}$ is called orthonormal if $\langle A_i,A_j\rangle =\delta_{ij}$, where  $\delta_{ij}$ is the Kronecker-delta.
\end{Def}
By identification of $\RR^{n_1\times \dots \times n_d}$ with $\RR^{N}$, where $N=\prod_{i=1}^d n_i$, the scalar product on tensors inherits all properties of the real scalar product.

\begin{Rem}\label{Rem:sctr}
It is seen by checking definitions that the scalar product on matrices is identical to the trace product, i.e., $\langle A,B\rangle =\Tr (A^\top B)$ for $A,B\in\RR^{m\times n}$.
\end{Rem}

An important property of the scalar product is compatibility with flattenings:
\begin{Lem}\label{Lem:orthflat}
Let $T_1,T_2\in \RR^{n_1\times n_2\times \dots \times n_d}$, let $\sigma$ be a $d$-to-$\tilde{d}$ flattening map. Then,
$$\langle T_1,T_2\rangle = \langle \sigma\dashv T_1,\sigma\dashv T_2\rangle.$$
In particular, $T_1$ and $T_2$ are orthogonal to each other if and only if $\sigma\dashv T_1$ and $\sigma\dashv T_2$ are.
\end{Lem}
\begin{proof}
A flattening is a bijection on the set of entries, therefore the result of the entry-wise scalar product is not changed by flattening.
\end{proof}

\begin{Prop}\label{Prop:prodorth}
Let $A^{(j)}_1,A^{(j)}_2\in \RR^{n^{(j)}_1\times \dots\times n^{(j)}_{c_j}}$, for $j=1,\dots, k$. Then,
$$\left\langle A^{(1)}_1\otimes \dots A^{(k)}_1, A^{(1)}_2\otimes \dots A^{(k)}_2\right\rangle = \prod_{j=1}^k\left\langle A^{(j)}_1,A^{(j)}_2\right\rangle.$$
In particular, if there exists $j$ such that $A^{(j)}_1,A^{(j)}_2$ are orthogonal to each other, then the outer products $A^{(1)}_1\otimes \dots \otimes A^{(k)}_1$ and $A^{(1)}_2\otimes \dots \otimes A^{(k)}_2$ are orthogonal to each other.
\end{Prop}
\begin{proof}
By performing induction on $k$, it suffices to prove the statement for $k=2$: Let $A_1,A_2\in \RR^{n_1\times \dots\times n_c}$ and $B_1,B_2\in \RR^{n_{c+1}\times \dots\times n_d}.$ Then,
$$\langle A_1\otimes B_1, A_2\otimes B_2\rangle = \langle A_1,B_1\rangle\cdot \langle A_2,B_2\rangle.$$
We proceed to prove this statement. Let $\sigma_1$ be the $c$-to-$1$-flattening, let $\sigma_2$ be the $(d-c)$-to-$1$-flattening. Let $v_i=\sigma_1\dashv A_i$, and $w_i=\sigma_2\dashv B_i$ for $i=1,2$. By Lemma~\ref{Lem:orthflat}, it holds that
$$\langle A_i,B_i\rangle = \langle v_i,w_i\rangle\quad\mbox{for}\;i=1,2.$$
Let $\tau$ be the $d$-to-$2$-flattening defined by $\tau:\{1,\dots, c\}\mapsto \{1\}, \{c+1,\dots, d\}\mapsto \{2\}$. Let $C_i=\tau\dashv (A_i\otimes B_i)$. By Lemma~\ref{Lem:orthflat}, it holds that
$$\langle A_1\otimes B_1, A_2\otimes B_2\rangle = \langle C_1,C_2\rangle.$$
By Lemma~\ref{Lem:prodflat}, it holds that
$$\langle C_1,C_2\rangle = \langle v_1\otimes w_1, v_2\otimes w_2\rangle.$$
Using that scalar product on tensors is the trace product (see~\ref{Rem:sctr}), we obtain
$$\langle v_1\otimes w_1, v_2\otimes w_2\rangle = \Tr(v_1w_1^\top w_2 v_2^\top).$$
The cyclic property of the trace product for matrices yields
$$\Tr(v_1w_1^\top w_2 v_2^\top) = \Tr(w_1^\top w_2 v_2^\top v_1) = w_1^\top w_2v_2^\top v_1 = \langle v_1,v_2\rangle\cdot \langle w_1,w_2\rangle.$$
All equalities put together yield the claim.
\end{proof}

\begin{Cor}
Let $\mu_1,\mu_2\in \RR^n$, and $d\in\NN$, such that $\langle \mu_1,\mu_2\rangle = 0$. Then,
$$\left\langle\mu_1^{\otimes d}, \mu_2^{\otimes d}\right\rangle = 0.$$
\end{Cor}

\begin{Def}\label{Def:odec}
Let $T\in \RR^{n_1\times n_2\times \dots \times n_d}$, let $[d]=S_1\cup S_2\cup\dots \cup S_k$ be a partition. A decomposition
$$T=\sum_{i=1}^r w_i\cdot A^{(1)}_i\otimes \dots \otimes A^{(k)}_i$$
with $w_i\in\RR$, and $A^{(j)}_i\in \RR^{\times_{\ell\in S_j} n_\ell}$, such that the set of $A^{(j)}_i$ with fixed $j$ is orthonormal, is called rank-$r$ \emph{orthogonal atomic decomposition} of $T$, with signature $(S_1,\dots, S_k)$. If $k=d$ and $S_i=\{i\}$, then the decomposition is called \emph{orthogonal CP-decomposition}.
\end{Def}

An orthogonal atomic decomposition does not need to exist necessarily. However, if it does, it is compatible with respect to flattenings, as Proposition~\ref{Prop:decompflat} will show. We introduce notation for a more concise statement of the compatibility first:

\begin{Def}
Let $(S_1,\dots, S_k)$ be a partition of $[d]$. We say a $d$-to-$\tilde{d}$-flattening  $\sigma$ is \emph{compatible} with the partition $(S_1,\dots, S_k)$, if it holds that $\{i,j\}\in S_\ell$ for some $\ell$ implies $\sigma(i)=\sigma(j)$. We say that $\sigma$ is \emph{strictly compatible} with the partition $(S_1,\dots, S_k)$, if it holds that $\{i,j\}\in S_\ell$ for some $\ell$ if and only if $\sigma(i)=\sigma(j)$.
\end{Def}

\begin{Prop}\label{Prop:decompflat}
Let $T\in \RR^{n_1\times n_2\times \dots \times n_d}$. Let
$$T=\sum_{i=1}^r w_i\cdot A^{(1)}_i\otimes \dots \otimes A^{(k)}_i$$
be an orthogonal atomic decomposition with signature $(S_1,\dots, S_k),$ let $\sigma$ be compatible with the signature. Then,
$$T=\sum_{i=1}^r w_i\cdot B^{(1)}_i\otimes \dots \otimes B^{(\tilde{d})}_i,\quad\mbox{where}\quad B^{(1)}_i=\sigma\dashv \left(\bigotimes_{j\in\sigma^{-1}(i)}B^{(1)}_j \right),$$
is an orthogonal atomic decomposition of $(\sigma\dashv T)$. In particular, if $\sigma$ is strictly compatible with the signature, then the decomposition is also an orthogonal CP-decomposition.
\end{Prop}
\begin{proof}
This is a direct consequence of Lemma~\ref{Lem:orthflat}, checking compatibility of scalar product and orthogonality with the flattening at each of the sets of indices $S_i$.
\end{proof}

\subsection{Identifiability of the Orthogonal Atomic Decomposition}
The orthogonal decomposition, as given in Definition~\ref{Def:odec}, does not need to exist for a tensor, nor does it need to be unique. We will show that due to the compatibility with flattenings, if it exists, it is unique, if the rank is chosen minimal.

The main ingredient, besides flattenings, is uniqueness of singular value decomposition~\cite{You36}, a classical result, which we state in a convenient form:
\begin{Thm}
Let $A\in\RR^{m\times n}$, let $r=\rank A$. Then, there is a singular value decomposition (= orthogonal CP-decomposition)
$$A=\sum_{i=1}^r w_i\cdot u_i\cdot v_i^\top\quad\mbox{with}\;u_i\in\RR^m, v_i\in \RR^n, w_i\in \RR$$
such that the $u_i$ are orthonormal, and the $v_j$ are orthonormal. In particular, there is no singular value decomposition of rank strictly smaller than $r$. Moreover, the singular value decomposition of $A$ is unique, up to:
\begin{description}
  \item[(a)] the sequence of summation, i.e., up to arbitrary permutation of the indices $i=1,\dots, n$
  \item[(b)] the choice of sign of $w_i,u_i,v_i$, i.e., up to changing the sign in any two of $w_i,u_i,v_i$ for fixed $i$
  \item[(c)] unitary transformations of the span of $u_i,u_j$ or $v_i,v_j$ such that $|w_i|=|w_j|$
\end{description}
Condition (c) includes (b) as a special case, and (c) can be removed as a condition if no two distinct $w_i,w_j$ have the same absolute value.
\end{Thm}

%\begin{Lem}
%Let $T\in \RR^{n_1\times n_2\times \dots \times n_d}$, and assume that $T$ has an orthogonal atomic decomposition with signature $(S_1,\dots, S_k)$ of rank $r$. Let $N_j=\prod_{i\in S_j} n_i$ for $j=1,\dots, k$. Then, $r\le N_j$ for all $j$.
%\end{Lem}
%\begin{proof}
%Fix some arbitrary $j$. Consider the $d$-to-$2$-flattening $\sigma: S_j\mapsto \{1\}, S_i\mapsto \{2\}$ for $i\neq j$. Let $m=N_j, n=\prod_{i\neq j} N_i$, and $A=\sigma\dashv T$. Note that $A$ is a $(m\times n)$-matrix. Let
%$$T=\sum_{i=1}^r w_i \cdot A^{(1)}_i\otimes \dots \otimes A^{(k)}_i$$
%be the orthogonal atomic decomposition of $T$, and let $u_i=\sigma\dashv A^{(j)}_i$, and $v_i=\sigma\dashv\bigotimes_{k\neq j} A^{(k)}_i$ for all $i$. Note that $u_i$ is an $m$-vector, and $v_j$ is an $n$-vector. By Proposition~\ref{Prop:decompflat},
%$$A=\sum_{i=1}^r w_i\cdot u_i\cdot v_i^\top$$
%is a singular value decomposition of $A$. In particular, the $u_i$ are a system of $r$ orthonormal vectors in $\RR^m$. Therefore, $r\le m=N_j$. Since $j$ was arbitrary, the statement follows.
%\end{proof}

\begin{Thm}\label{Thm:uniq}
Let $T\in \RR^{n_1\times n_2\times \dots \times n_d}$, and assume that $T$ has an orthogonal atomic decomposition
$$T=\sum_{i=1}^r w_i\cdot A^{(1)}_i\otimes \dots \otimes A^{(k)}_i$$
of signature $(S_1,\dots, S_k)$, such that $w_i\neq 0$ for all $i$. Then:
\begin{description}
  \item[(i)] Denote $N_j=\prod_{i\in S_j} n_i$ for $j=1,\dots, k$. Then, $r\le N_j$ for all $j$.
  \item[(ii)] There is no orthogonal atomic decomposition of $T$ with signature $(S_1,\dots, S_k)$, and of rank strictly smaller than $r$.
  \item[(iii)] The orthogonal atomic decomposition of $T$ of rank $r$ is unique, up to:
\begin{description}
  \item[(a)] the sequence of summation, i.e., up to arbitrary permutation of the indices $i=1,\dots, n$
  \item[(b)] the choice of sign of $w_i,A^{(k)}_i$, i.e., up to changing the sign in any two of $w_i$ and the $A^{(k)}_i$ for fixed $i$ and arbitrary $k$
  \item[(c)] transformations of factors $A^{(k)}_i,A^{(k)}_j,$ and their respective tensor products, such that $|w_i|=|w_j|$, which induce unitary transformations in all flattenings compatible with the signature $(S_1,\dots, S_k)$.
\end{description}
Condition (c) includes (b) as a special case, and (c) can be removed as a condition if no two distinct $w_i,w_j$ have the same absolute value.

\end{description}
\end{Thm}
\begin{proof}
Fix some arbitrary $j$. Consider the $d$-to-$2$-flattening $\sigma: S_j\mapsto \{1\}, S_i\mapsto \{2\}$ for $i\neq j$, note that $\sigma$ is compatible with the signature. Let $m=N_j, n=\prod_{i\neq j} N_i$, and $A=\sigma\dashv T$. Note that $A$ is a $(m\times n)$-matrix. Let
$$T=\sum_{i=1}^r w_i \cdot A^{(1)}_i\otimes \dots \otimes A^{(k)}_i$$
be the orthogonal atomic decomposition of $T$, and let $u_i=\sigma\dashv A^{(j)}_i$, and $v_i=\sigma\dashv\bigotimes_{k\neq j} A^{(k)}_i$ for all $i$. Note that $u_i$ is an $m$-vector, and $v_j$ is an $n$-vector. By Proposition~\ref{Prop:decompflat},
$$A=\sum_{i=1}^r w_i\cdot u_i\cdot v_i^\top$$
is a singular value decomposition of $A$.\\

(i) In particular, the $u_i$ are a system of $r$ orthonormal vectors in $\RR^m$. Therefore, $r\le m=N_j$. Since $j$ was arbitrary, statement (i) follows.\\
(ii) Since the $w_i$ are non-zero, it holds that $\rank A = r$. Would there be an orthogonal atomic decomposition of $T$ with signature $(S_1,\dots, S_k)$ of rank strictly smaller than $r$, there would be a singular value decomposition of $A$ of rank strictly smaller than $r$, contradicting Proposition~\ref{Prop:decompflat}.\\
(iii) Observe that the flattening by $\sigma$ induces a bijection between the orthogonal atomic decompositions of $T$, of rank $r$, and the singular value decompositions of $A$, of rank $r$. The statement in (iii) then follows directly from the uniqueness asserted in Proposition~\ref{Prop:decompflat} for the singular value decomposition of $A$.
\end{proof}

Again, we would like to stress that the present orthogonal decomposition model is different from the one in~\cite{Kolda01orthogonaltensor}; ours being factor-wise orthogonal between different summands, while the orthogonal rank decomposition in~\cite{Kolda01orthogonaltensor} being summand-wise orthogonal, and the combinatorial orthogonal rank decomposition enforcing orthogonality of factors in the same summand. Therefore, Theorem~\ref{Thm:uniq} does not contradict Lemma~3.5 in~\cite{Kolda01orthogonaltensor}.\\

Another result which seems to be folklore, but not available in the literature, is that it is a strong restruction for a tensor to assume that it has an orthogonal decomposition. Since it is almost implied by the identifiability Theorem~\ref{Thm:uniq}, we state a quantitative version of this:

\begin{Prop}\label{Prop:notallorth}
The set of tensors $T\in \RR^{n_1\times n_2\times \dots \times n_d}$, with $d\ge 3$, and $n_j\ge 2$ for all $j$, for which $T$ has an orthogonal CP-decomposition, is a Lebesgue zero set.
\end{Prop}
\begin{proof}
The CP-decomposition can be viewed as an algebraic map
\begin{align*}
\phi: \RR\times \RR^{n_1}\times\dots\times\RR^{n_d}&\rightarrow \RR^{n_1\times n_2\times \dots \times n_d}\\
(w_i,v^{(j)}_i) &\mapsto \sum_{i=1}^r w_i\cdot v^{(1)}_i\otimes \dots \otimes v^{(k)}_i.
\end{align*}
Since the left hand side is an irreducible variety, the image of the map $\phi$ also is. The orthogonal CP-decompositions form an algebraic subset of the left hand side. Therefore the state follows from the fact that $\phi$ is not surjective. This follows from a degree of freedom resp.~dimension count. One has
\begin{align*}
D_1&:=\dim\RR^{n_1\times n_2\times \dots \times n_d} =\prod_{i=1}^d n_d,\quad\mbox{and }\\
D_2&:=\dim (\RR\times \RR^{n_1}\times\dots\times\RR^{n_d})^r = r\cdot (n_1+\dots + n_d+1).
\end{align*}
Theorem~\ref{Thm:uniq}~(i) implies
$$D_2\le n_j\cdot (n_1+\dots+n_k+1).$$
An explicit computation shows that $D_1\gneq D_2$, which proves the statement.\\

The proof above can be rephrased in terms of the CP-rank (see~\cite{Cat2002} for an introduction), can be obtained by observing that the generic CP-rank of the tensors in questions must be strictly larger than $\min (n_1,\dots, n_k)$, then proceeding again by arguing that the algebraic set of tensors with orthogonal CP-decompositions must be a proper subset of all tensors with that format, thus a Lebesgue zero set.
\end{proof}

Proposition~\ref{Prop:notallorth} can be extended to orthogonal atomic decompositions with signature $(S_1,\dots, S_k), k\ge 3$, by considering suitable unflattenings.

\subsection{Tensors and Moments}
We briefly show how tensors relate to moments of multivariate real random variables:

\begin{Def}
Let $X$ be a real $n$-dimensional random variable. Then, define:
\begin{align*}
\mbox{the characteristic function of}\; X\;\mbox{as}\quad &\quad\quad 	\varphi_X(\tau) := \EE \left[ \exp \left(i \tau X\right) \right], \\
\mbox{the moment generating function of}\; X\mbox{as}\quad &\quad\quad 	\chi_X(\tau) := \log\EE \left[ \exp \left(i \tau X\right) \right],
\end{align*}
where $\tau\in \RR^{1\times n}$ is a formal vector of variables. The $d$-th \emph{moment} (or moment tensor) $\mom_d(X)\in \RR^{n^{(\times d)}}$ of $X$, and the $d$-th \emph{cumulant} (or cumulant tensor) $\cumu_d(X)\in \RR^{n^{(\times d)}}$ of $X$ are defined\footnote{in case of convergence} as the coefficients in the multivariate Taylor expansions
\begin{align*}
\varphi_X(\tau) & = \sum_{d=1}^\infty \left(i \tau\right) \circ \frac{\mom_d(X)}{d!},\\
\chi_X(\tau) & = \sum_{d=1}^\infty \left(i \tau\right) \circ \frac{\cumu_d(X)}{d!}.
\end{align*}
\end{Def}

In the following, we will always assume that the moments and cumulants in question exist.

The moments and cumulants of a linearly transformed random variable are the multilinearly transformed
moments.
\begin{Prop}\label{Prop:lintrans}
Let $X$ be a real $n$-dimensional random variable and let $A \in \RR^{m \times n}.$ Then,
\begin{align*}
	\mom_d(A\cdot X) &= A  \circ \mom_d(X),\\
	\cumu_d(A\cdot X) &= A  \circ \cumu_d(X).
\end{align*}
\end{Prop}
\begin{proof}
We prove the statement for cumulants, the proof for moments is completely analogous.
For the cumulant generating functions $\chi_X$ of $X$ and $\chi_{A\cdot X}$ of $A\cdot X$, it holds that
\begin{align*}
	\chi_{A\cdot X}(\tau) & = \EE \left[ \exp \left(i \tau\cdot A\cdot X\right) \right] \\
                          & = \EE \left[ \exp \left(i \left(\tau\cdot A \right) \cdot X\right) \right]\\
                          & = \sum_{d=1}^\infty \left(i \tau\right)\circ \left(A \circ \frac{\mom_d(X)}{d!}\right).
\end{align*}
The last equality follows from the definition of $\chi_X(\tau)$. But by definition, it also holds that
\begin{align*}
	\chi_{A\cdot X}(\tau)        & = \sum_{d=1}^\infty \left(i \tau\right) \circ \frac{\mom_d(A\cdot X)}{d!},
\end{align*}
therefore the statement follows from comparing coefficient tensors.
\end{proof}

\section{Relation to Mixture Models}
\label{sec:estimation}
\subsection{The Estimation Problem}
Throughout the paper, we will consider the following independent rank $1$ mixture model:\\

{\bf Generative Model:} $X_1,\dots, X_r$ are independent, $\RR^n$-valued random variables, with $r\le n$, and probability/mass density functions $X_i\sim p_i$. Let $w_1,\dots, w_r\in\RR$ be arbitrary such that $\sum_{i=1}^r w_i = 1$, and let $Y\sim \sum_{i=1}^r w_r p_i$ be the corresponding mixture of the $X_i$. Assume that there are $\mu_1,\dots, \mu_r\in\RR^n$ with $\|\mu_i\|_2=1$, and random variables $Z_i\in\RR$, such that $X_i=\mu_i\cdot Z_i$. Assume that the $\mu_i$ are linearly independent, and $\mom_d(Z_i)=1$ for $d=2,\dots, m$.\\

{\bf Estimation Task:} Given $\mom_2 (Y),\mom_3 (Y),\dots,\mom_m (Y), m\ge 3$, or estimators thereof, determine/estimate $\mu_i$ and $w_i$ for $i=1,\dots, r$.\\

While the above scenario seems very restrictive, several important problems can be reduced to this setting, see for example~\cite{Hsu2012}, or chapter 3 of~\cite{Ana2012}. We recommend the interested reader to read the exposition there.

\subsection{Algebraic Formulation via Moments}
The estimation problem presented above can be reformulated as a purely algebraic problem, see~\cite{Ana2012}. Namely, the $\mom_i$ are explicitly calculable in terms of the expectations of the $\mu_i$ and $w_i$. Then, Proposition~\ref{Prop:lintrans} implies that $\mom_d(X_i)= \mu_i^{\otimes d}$ for all $d$, therefore $\mom_d(Y)=\sum_{i=1}^r w_i\cdot \mu_i^{\otimes d}$ for all $d$, thus yielding the following algebraic version of the estimation problem.\\

{\bf Algebraic Problem:} Let $r\le n$, let $\mu_1,\dots, \mu_r\in\RR^n$ be linearly independent and $w_1,\dots, w_r\in\RR$ arbitrary such that $\sum_{i=1}^r w_i = 1$. Given (exact or noisy estimators for)
$$\mom_d = \sum_{i=1}^r w_i\cdot \mu_i^{\otimes d}\quad\mbox{for}\;d=2,\dots, m,\;\mbox{with}\;m\ge 3,$$
determine the $\mu_i$ and $w_i$.\\

\section{Algorithms}

\subsection{Orthogonal Decomposition of Tensors}
A special case of orthogonal decomposition is singular value decomposition (SVD). There are a huge amount of well-studied methods for obtaining the singular value decomposition, which we will not discuss. However, we will make extensive use of the SVD algorithm, as described in Algorithm~\ref{Alg:SVD} as a black box.
\begin{algorithm}[ht]
\caption{\label{Alg:SVD} \texttt{SVD}. Singular Value Decomposition of Matrices.\newline
\textit{Input:} A matrix $A\in\RR^{m\times n}$.
\textit{Output:} The singular value decomposition $A= U\cdot \Sigma\cdot V^\top$, with $U\in\RR^{m\times r}, V\in\RR^{n\times r}$ orthogonal, $\Sigma\in\RR^{r\times r}$ diagonal, and the rank $r=\rank A$}
\end{algorithm}

First, for completeness, we treat the trivial case in Algorithm~\ref{Alg:orthdecomp1}.

\begin{algorithm}[ht]
\caption{\label{Alg:orthdecomp1} \texttt{OTD1}. Orthogonal Tensor Decomposition in one factor.\newline
\textit{Input:} A tensor $T\in\RR^{n_1\times\dots\times n_d}$, a signature $(S_1)$.
\textit{Output:} The orthogonal atomic decomposition $T=\sum_{i=1}^r w_i\cdot A_i$.}
\begin{algorithmic}[1]
    \State Return rank $r=1,$ coefficients $w_1=\|T\|,$ factors $A_1=\|T\|^{-1}\cdot T$.
\end{algorithmic}
\end{algorithm}

Now we explicitly describe how to compute the orthogonal decomposition if each summand has two tensor factors. Algorithm~\ref{Alg:orthdecomp2} computes the decomposition by a proper reformatting of the entries, computing the singular value decomposition, then reformatting again.

\begin{algorithm}[ht]
\caption{\label{Alg:orthdecomp2} \texttt{OTD2}. Orthogonal Tensor Decomposition in two factors.\newline
\textit{Input:} A tensor $T\in\RR^{n_1\times\dots\times n_d}$, a signature $(S_1, S_2)$.
\textit{Output:} The orthogonal atomic decomposition $T=\sum_{i=1}^r w_i\cdot A_i\otimes B_i$ (assumed to exist), including the rank $r$}
\begin{algorithmic}[1]
    \State Define $\sigma: [d]\rightarrow [2], S_i\mapsto \{i\}.$
    \State Set $A\leftarrow (\sigma\dashv T)$. Note that $A\in\RR^{m\times n}$, with $m=\prod_{i\in S_1}n_i, n=\prod_{i\in S_2}n_i.$
    \State Compute the \texttt{SVD} of $A=U\cdot \Sigma\cdot V^\top$, see Algorithm~\ref{Alg:SVD}.
    \State Return rank $r=\rank A$.
    \State Return coefficients $w_i=\Sigma_{ii}$ for $i=1,\dots, r$.
    \State For all $i$, let $U_i$ be the $i$-th column of $U$, let $V_i$ be the $i$-th columns of $V$.
    \State Return factors $A_i=\sigma\vdash U_i, B_i =\sigma\vdash V_i$ for $i=1,\dots, r$.
\end{algorithmic}
\end{algorithm}

The algorithm for the general case, Algorithm~\ref{Alg:orthdecomp}, consists as well of repeated applications of reindexing and singular value decomposition. Variants of singular value decomposition exist with adjustable noise tolerance or singular value thresholding, and can therefore be employed to obtain thresholding and numerically stable variants of Algorithm~\ref{Alg:orthdecomp}. Furthermore, step~\ref{Alg:orthdecomp-step1} allows for an arbitrary choice of $k$-to-$2$-flattening, in each recursion. Since in the presence of noise, the results might differ when taking a different sequence flattenings, the numerical stability can be improved by clustering the results of all possible choices, then averaging.

\begin{algorithm}[ht]
\caption{\label{Alg:orthdecomp} \texttt{OTD}. Orthogonal Tensor Decomposition.\newline
\textit{Input:} A tensor $T\in\RR^{n_1\times\dots\times n_d}$, a signature $(S_1, \dots, S_k)$.
\textit{Output:} The orthogonal atomic decomposition $T=\sum_{i=1}^r w_i\cdot A^{(1)}_i\otimes \dots \otimes A^{(k)}_i$ (assumed to exist), including the rank $r$}
\begin{algorithmic}[1]
    \State \label{Alg:orthdecomp-step1} Choose any $k$-to-$2$-flattening map $\tau$.
    \State Set $\tilde{S}_j\leftarrow \cup_{i\in\tau^{-1}(j)}S_i$ for $j=1,2$.
    \State Set $\tilde{T}\leftarrow \tau\dashv T$.
    \State Use \texttt{OTD2}, Algorithm~\ref{Alg:orthdecomp2}, to compute the orthogonal atomic decomposition
   $\tilde{T}=\sum_{i=1}^r w_i\cdot A_i\otimes B_i$
     with signature $(\tilde{S}_1,\tilde{S}_2)$.
    \State Return the $w_i$ as coefficients and $r$ as the rank for the decomposition of $T$.
    \State \label{Alg:orthdecomp-step6} For $i=1,\dots, r$, use the suitable one of \texttt{OTD1},\texttt{OTD2},\texttt{OTD}, i.e., Algorithm~\ref{Alg:orthdecomp1},\ref{Alg:orthdecomp2}, or~\ref{Alg:orthdecomp}, to compute the orthogonal atomic decomposition $(\tau\vdash A_i)=\sum_{i=1}^1 1\cdot\bigotimes_{\tau(j)\in S_1} A^{(j)}_i$, noting that rank is one, and using the signature $(S_j\;:\;\tau(j)\in \tilde{S}_1).$
    \State \label{Alg:orthdecomp-step7} For $i=1,\dots, r$, use the suitable one of \texttt{OTD1},\texttt{OTD2},\texttt{OTD}, i.e., Algorithm~\ref{Alg:orthdecomp1},\ref{Alg:orthdecomp2}, or~\ref{Alg:orthdecomp}, to compute the orthogonal atomic decomposition $(\tau\vdash B_i)=\sum_{i=1}^1 1\cdot\bigotimes_{\tau(j)\in S_2} A^{(j)}_i$, noting that rank is one, and using the signature $(S_j\;:\;\tau(j)\in \tilde{S}_2).$
    \State Return the $A^{(j)}_i$ as factors for $T$.
\end{algorithmic}
\end{algorithm}

Termination of Algorithm~\ref{Alg:orthdecomp} is implied by the observation that in each recursion, the partition of $[d]$ is made strictly finer. Since $[d]$ has finite cardinality, there is only a finite number of recursions. The fact that the decompositions in steps~\ref{Alg:orthdecomp-step6} and~\ref{Alg:orthdecomp-step7} have rank one, and coefficients $1$, follows from the uniqueness of the orthogonal decomposition guaranteed in Theorem~\ref{Thm:uniq}. Correctness of Algorithm~\ref{Alg:orthdecomp} follows from repeated application of Proposition~\ref{Prop:decompflat}, and the uniqueness of singular value decomposition.

\subsection{An Estimator for the Mixture Model}
For illustrative purposes, we write out Algorithm~\ref{Alg:orthdecomp} for the problem introduced in~\ref{sec:estimation}, which has also extensively been studied in~\cite{Ana2012}:

{\bf Example:} Let $r\le n$, let $\mu_1,\dots, \mu_r\in\RR^n$ be linearly independent and $w_1,\dots, w_r\in\RR$ arbitrary such that $\sum_{i=1}^r w_i = 1$. Given (exact or noisy estimators for)
$$\mom_d = \sum_{i=1}^r w_i\cdot \mu_i^{\otimes d}\quad\mbox{for}\;d=2,3,$$
determine the $\mu_i$ and $w_i$.\\

Algorithm~\ref{Alg:degree3} solves the problem, by reducing it to

\begin{algorithm}[ht]
\caption{\label{Alg:degree3} Model identification.\newline
\textit{Input:} $\mom_2,\mom_3$
\textit{Output:} $w_1,\dots, w_r, \mu_1,\dots,\mu_r$.}
\begin{algorithmic}[1]
    \State Set $r\leftarrow \rank(\mom_2)$.
    \State Compute the SVD\footnote{Note: since $\mom_2$ is symmetric, the SVD also is.} $\mom_2=U\cdot \Sigma\cdot U^\top$.
    \State Set $W\leftarrow U\cdot \Sigma^{-\frac{1}{2}}$.
    \State \label{Alg:degree3-step4}Set $T:=W^\top\circ \mom_3$.
    \State Define the flattening map $\sigma: 1\mapsto 1,2\mapsto 2, 3\mapsto 2$.
    \State Set $\tilde{T}:=\sigma\dashv T.$
    \State Compute the rank $r$ SVD $\tilde{T}=\sum_{i=1}^r \tilde{w}_i\cdot \tilde{\mu}^{(1)}_i\cdot v_i^\top$.
    \State \label{Alg:degree3-step8} Return $w_i= \tilde{w}_i^{-2}$ for $i=1,\dots, r$.
    \State Set $\tilde{A}_i=(\sigma\vdash v_i)$ for $i=1,\dots, r$.
    \State \label{Alg:degree3-step10} Compute the rank $1$ SVD $\tilde{A}_i= \tilde{\mu}^{(2)}_i\cdot \left(\tilde{\mu}^{(3)}_i\right)^\top$.
    \State \label{Alg:degree3-step11} Set $\tilde{\mu}_i\leftarrow \frac{1}{3}\left(\tilde{\mu}^{(1)}_i+\tilde{\mu}^{(2)}_i+\tilde{\mu}^{(3)}_i\right),$ for $i=1,\dots, r$.
    \State \label{Alg:degree3-step12} Compute the pseudo-inverse $B$ of $W'$. Return $\mu_i = B\cdot \tilde{\mu}_i\cdot \tilde{w}_i$, for $i=1,\dots, r.$
\end{algorithmic}
\end{algorithm}
Theorem 4.3 in~\cite{Ana2012} implies that the tensor $T$ obtained in step~\ref{Alg:degree3-step4} has an orthogonal CP-decomposition, and it implies the correctness of steps~\ref{Alg:degree3-step8}, and~\ref{Alg:degree3-step12}. The fact that $\tilde{A}_i$ in step~\ref{Alg:degree3-step8} has rank one, and the coefficients are $1$, follow from the uniqueness of the decomposition guaranteed in Theorem~\ref{Thm:uniq}.

Note that explicit presentation of the algorithm could be substantially abbreviated by applying $\texttt{ODT}$ directly to $\tilde{T}$ in step~\ref{Alg:degree3-step4}, with signature $(\{1\},\{2,3\})$, and then performing the analogues of steps~\ref{Alg:degree3-step8} and~\ref{Alg:degree3-step12}. Furthermore, the accuracy of the estimator in step~\ref{Alg:degree3-step11} can be improved, by repeating the procedure for the three possible signatures $(\{1\},\{2,3\}), (\{2\},\{1,3\}),$ and $(\{3\},\{1,3\})$, then averaging, or weighted averaging, over the nine estimates for each $\tilde{\mu}_i$, making use of the symmetry of the problem.\\

Also, similar to Algorithm~\ref{Alg:orthdecomp}, the presented Algorithm~\ref{Alg:degree3}, while already numerically stable, can be modified to cope better with noise by, e.g., introducing thresholding to the singular value decomposition and rank computations. The numerical stability with respect to noise is governed by the numerical stability of the SVDs performed, and the pseudo-inversion of $W'$ in step~\ref{Alg:degree3-step12}.\\

Algorithm~\ref{Alg:degree3} is also related to Algorithm~1 proposed in ~\cite{Ana2012NIPS}. Namely, $\mbox{Triples}(\eta)$, as defined in section~4.1, is a degree $2$-projection of the tensor $T$, and therefore can be also understood as a random projection of the flattening $\sigma\dashv T$.\\

Furthermore, an estimator for the hierarchical models described in~\cite{Ish2013ICML,Son2013ICML} can be constructed in a similar way.

%\subsection{Higher Order Moments, and the Non-Symmetric Case}

\section{Conclusion}
We have demonstrated that computing the orthogonal decomposition of an arbitrary degree tensor, symmetric or not, can be reduced to a series of singular value decompositions, and we have described efficient algorithms to do so. This makes orthogonal tensor decomposition approachable by the wealth of theoretical results and existing methods for eigenvalue problems and singular value decomposition. Moreover, we have exemplified our method in the case of identifying components in a low-rank mixture model.

\section*{Acknowledgments}
I thank Arthur Gretton, Zolt\'an Szab\'o, and Andreas Ziehe for interesting discussions. I thank the Mathematisches Forschungsinstitut Oberwolfach for support.

\bibliographystyle{plainnat}
\bibliography{mixture}
\end{document}

%% file: macros.tex
\usepackage[utf8]{inputenc}

\usepackage{color}
\usepackage{subfigure}
\usepackage{url}
\usepackage{graphicx}
\usepackage{amsmath,amsthm}
\usepackage{bm,bbm}
\usepackage{algorithm,algpseudocode}
\usepackage[plainpages=false]{hyperref}

\usepackage{titlesec}
\titleformat{\section}
	{\normalfont\Large\bfseries\filcenter}{\thesection.}{1 ex}{}
\titleformat{\subsection}%[runin]
	{\normalfont\normalsize\bfseries}{\thesubsection.}{1 ex}{}
\titleformat{\subsubsection}[runin]
	{\normalfont\normalsize\bfseries\filcenter}{\thesubsubsection.}{1 ex}{}

\usepackage[charter]{mathdesign}
\usepackage[mathcal]{eucal}
\usepackage[margin=1.2 in]{geometry}

\usepackage[square,numbers,sort&compress]{natbib}

\newcommand{\emailaddr}[1]{\href{mailto:#1}{\texttt{#1}}}

\usepackage{thmtools,thm-restate}
\declaretheoremstyle[qed=$\diamond$,headpunct={ --- },headfont=\normalfont\itshape]{myremark}
\declaretheoremstyle[qed=$\blacksquare$,bodyfont=\normalfont]{mydefinition}
\declaretheorem[name=Theorem]{Thm}
\declaretheorem[within=section,name=Lemma]{Lem}
\declaretheorem[sibling=Lem,name=Proposition]{Prop}
\declaretheorem[sibling=Lem,name=Corollary]{Cor}

\declaretheorem[style=myremark,sibling=Lem,name=Remark]{Rem}
\declaretheorem[style=mydefinition,sibling=Lem,name=Definition]{Def}

\declaretheorem[style=mydefinition,sibling=Lem,name=Example]{Ex}

\usepackage{enumerate}
\usepackage{fancyhdr}
\usepackage{verbatim}
\usepackage{array}
\usepackage{fullpage}
\usepackage{enumitem}

\newcommand{\rank}{\operatorname{rank}}
\newcommand{\Tr}{\operatorname{Tr}}

\newcommand{\mom}{\ensuremath{\mathbf{M}}}
\newcommand{\cumu}{\ensuremath{\kappa}}

\usepackage{multirow}
\usepackage{tikz}
\usetikzlibrary{matrix, shapes, arrows, calc, backgrounds}
\usetikzlibrary{arrows,chains,matrix,positioning,scopes}
\makeatletter
\tikzset{join/.code=\tikzset{after node path={%
\ifx\tikzchainprevious\pgfutil@empty\else(\tikzchainprevious)%
edge[every join]#1(\tikzchaincurrent)\fi}}}
\tikzset{>=stealth',every on chain/.append style={join},
         every join/.style={->}}

\newcommand{\RR}{\ensuremath{\mathbb{R}}}

\newcommand{\NN}{\ensuremath{\mathbb{N}}}
\newcommand{\EE}{\ensuremath{\mathbb{E}}}

%% file: mixture.bbl
\begin{thebibliography}{18}
\providecommand{\natexlab}[1]{#1}
\providecommand{\url}[1]{\texttt{#1}}
\expandafter\ifx\csname urlstyle\endcsname\relax
  \providecommand{\doi}[1]{doi: #1}\else
  \providecommand{\doi}{doi: \begingroup \urlstyle{rm}\Url}\fi

\bibitem[Anandkumar et~al.(2012{\natexlab{a}})Anandkumar, Foster, Hsu, Kakade,
  and Lu]{Ana2012NIPS}
Anima Anandkumar, Dean~P. Foster, Daniel Hsu, Sham~M. Kakade, and Yi-Kai Lu.
\newblock A spectral algorithm for latent dirichlet allocation.
\newblock \emph{ArXiv e-print}, 2012{\natexlab{a}}.

\bibitem[Anandkumar et~al.(2012{\natexlab{b}})Anandkumar, Ge, Hsu, Kakade, and
  Telgarsky]{Ana2012}
Anima Anandkumar, Rong Ge, Daniel Hsu, Sham~M. Kakade, and Matus Telgarsky.
\newblock Tensor decomposition for learning latent variable models.
\newblock \emph{ArXiv e-print}, 2012{\natexlab{b}}.

\bibitem[Bro(1997)]{Bro1997}
Rasmus Bro.
\newblock {PARAFAC}. tutorial and applications.
\newblock \emph{Chemometrics and Intelligent Laboratory Systems}, 38\penalty0
  (2):\penalty0 149 -- 171, 1997.

\bibitem[Carroll and Chang(1970)]{CarrollChang1970}
J.~Douglas Carroll and Jih-Jie Chang.
\newblock Analysis of individual differences in multidimensional scaling via an
  n-way generalization of 'eckart–young' decomposition.
\newblock \emph{Psychometrika}, 35:\penalty0 283--319, 1970.

\bibitem[Catalisano et~al.(2002)Catalisano, Germatia, and Gimigliano]{Cat2002}
Maria~V. Catalisano, Anthony~V. Germatia, and Allesandro Gimigliano.
\newblock Rank of tensors, secant varieties of {S}egre varieties and fat
  points.
\newblock \emph{Linear Algebra and its Applications}, pages 263--285, 2002.

\bibitem[De~Lathauwer et~al.(2008)De~Lathauwer, Comon, and
  Mastronardi]{DeLathauwer:2008}
Lieven De~Lathauwer, Pierre Comon, and Nicola Mastronardi.
\newblock Special issue on tensor decompositions and applications.
\newblock \emph{SIAM Journal on Matrix Analysis and Applications}, 30\penalty0
  (3):\penalty0 .7--.7, September 2008.
\newblock ISSN 0895-4798.

\bibitem[Denis and Dhorne(1989)]{Deni1989}
J.~B. Denis and T.~Dhorne.
\newblock {O}rthogonal tensor decomposition of 3-way tables.
\newblock In R.~Coppi and S.~Bolasco, editors, \emph{{M}ultiway data
  analysis.}, pages 31--37. Elsevier, Amsterdam, 1989.

\bibitem[Harshman(1970)]{Harshman1970}
Richard~A. Harshman.
\newblock Foundations of the parafac procedure: Models and conditions for an
  "explanatory" multi-modal factor analysis.
\newblock \emph{UCLA Working Papers in Phonetics}, 16\penalty0 (84):\penalty0
  1--84, 1970.

\bibitem[Hitchcock(1927)]{Hitchcock1927}
Frank~L. Hitchcock.
\newblock The expression of a tensor or a polyadic as a sum of products.
\newblock \emph{Journal of Mathematics and Physics}, 6:\penalty0 164--189,
  1927.

\bibitem[Hsu and Kakade(2012)]{Hsu2012}
Daniel Hsu and Sham~M. Kakade.
\newblock Learning mixtures of spherical {G}aussians: Moment methods and
  spectral decompositions.
\newblock \emph{ArXiv e-print}, 2012.

\bibitem[Ishteva et~al.(2013)Ishteva, Park, and Song]{Ish2013ICML}
Maria Ishteva, Haeson Park, and Le~Song.
\newblock Unfolding latent tree structures using 4th order tensors.
\newblock \emph{JMLR Workshop and Conference Proceedings, ICML 2013}, pages
  316--324, 2013.

\bibitem[Kolda(2001)]{Kolda01orthogonaltensor}
Tamara~G. Kolda.
\newblock Orthogonal tensor decompositions.
\newblock \emph{SIAM Journal on Matrix Analysis and Applications}, 23\penalty0
  (1):\penalty0 243--255, 2001.

\bibitem[Martin and Loan(2006)]{Martin06ajacobi-type}
Carla D.~Moravitz Martin and Charles F.~Van Loan.
\newblock A jacobi-type method for computing orthogonal tensor decompositions.
\newblock \emph{SIAM Journal on Matrix Analysis and Applications}, pages
  1219--1232, 2006.

\bibitem[Nion and Sidiropoulos(2009)]{NionSidi2009}
Dimitri Nion and Nikos~D. Sidiropoulos.
\newblock A {PARAFAC}-based technique for detection and localization of
  multiple targets in a mimo radar system.
\newblock In \emph{Proc. ICASSP '09}, pages 2077--2080, 2009.

\bibitem[Sidiropoulos(2004)]{Sidi2004}
Nikos~D. Sidiropoulos.
\newblock Low-rank decomposition of multi-way arrays: a signal processing
  perspective.
\newblock In \emph{Sensor Array and Multichannel Signal Processing Workshop
  Proceedings}, pages 52 -- 58, 2004.

\bibitem[{Sidiropoulos} et~al.(2000){Sidiropoulos}, {Bro}, and
  {Giannakis}]{Parafac2000}
Nikos~D. {Sidiropoulos}, Rasmus {Bro}, and Georgios~B. {Giannakis}.
\newblock {Parallel factor analysis in sensor array processing}.
\newblock \emph{IEEE Transactions on Signal Processing}, 48:\penalty0
  2377--2388, August 2000.

\bibitem[Song et~al.(2013)Song, Ishteva, Parikh, Xing, and Park]{Son2013ICML}
Le~Song, Maria Ishteva, Ankur Parikh, Eric Xing, and Haeson Park.
\newblock Hierarchical tensor decomposition of latent tree graphical models.
\newblock \emph{JMLR Workshop and Conference Proceedings, ICML 2013}, pages
  334--342, 2013.

\bibitem[Young and Eckart(1936)]{You36}
G.~Young and C.~Eckart.
\newblock The approximation of one matrix by another of lower rank.
\newblock \emph{Psychometrika}, 1:\penalty0 211--218, 1936.

\end{thebibliography}
